\documentclass[letterpaper]{article}
\usepackage{aaai}
\usepackage{times}
\usepackage{helvet}
\usepackage{courier}

\usepackage{amssymb}
\usepackage[dvips]{graphicx}
\usepackage{psfrag}
\usepackage[linesnumbered,ruled,vlined]{algorithm2e}
\usepackage{multirow}
\usepackage{amsthm}
\newcommand{\brac}[1]{\{\!\!\{#1\}\!\!\}}
\newcommand{\br}[1]{\{#1\}}
\newcommand{\ignore}[1]{}
\newcommand{\lbr}[1]{\langle{#1}\rangle}

\newcommand{\LL}[0]{\preceq_{ll}}
\newcommand{\LC}[0]{\preceq_{lc}}
\newcommand{\VL}[0]{\preceq_{vl}}
\newcommand{\VC}[0]{\preceq_{vc}}
\newcommand{\VLL}[0]{\preceq_{vll}}
\newcommand{\VLC}[0]{\preceq_{vlc}}
\newcommand{\LVL}[0]{\preceq_{lvl}}
\newcommand{\LVC}[0]{\preceq_{lvc}}

\newtheorem{theorem}{Theorem}
\newtheorem{definition}{Definition}
\newtheorem{proposition}{Proposition}

\begin{document}
%
\title{A Comparison of Lex Bounds for Multiset Variables in Constraint Programming}
\author{Y.C.~Law \and J.H.M.~Lee\and M.H.C.~Woo \\
Department of Computer Science and Engineering\\
The Chinese University of Hong Kong, \\
Shatin, N.T., Hong Kong\\
\{yclaw,jlee,hcwoo\}@cse.cuhk.edu.hk
\And
T.~Walsh \\
NICTA \& UNSW \\ 
Sydney, Australia\\
toby.walsh@nicta.com.au
}
\ignore{
\author{
Y.C.~Law$^1$ \and J.H.M.~Lee$^1$ \and T.~Walsh$^2$ \and M.H.C.~Woo$^1$ \\ \\
$^1$ Department of Computer Science and Engineering\\
The Chinese University of Hong Kong, Shatin, N.T., Hong Kong\\
\{yclaw,jlee,hcwoo\}@cse.cuhk.edu.hk\\ \\
$^2$ NICTA \& UNSW\\
Sydney, Australia\\
toby.walsh@nicta.com.au
}}

\maketitle

\begin{abstract}
Set and multiset variables in constraint programming have typically 
been represented using subset bounds. However, this is a 
weak representation that neglects potentially useful information 
about a set such as its cardinality.  For set
variables, the length-lex (LL)  representation successfully provides 
information about the length (cardinality) and position in the 
lexicographic ordering.  For multiset variables, where elements
can be repeated, we consider richer representations that
take into account additional information. 
We study eight different representations 
in which we maintain bounds according to one of the 
eight different orderings:\ length-(co)lex (LL/LC), variety-(co)lex 
(VL/VC), length-variety-(co)lex (LVL/LVC), and variety-length-(co)lex 
(VLL/VLC) orderings.  These representations integrate together 
information about the cardinality, variety
(number of distinct elements in the multiset), and position in some
total ordering.  Theoretical and empirical comparisons of 
expressiveness and compactness of the eight representations suggest 
that length-variety-(co)lex (LVL/LVC) and variety-length-(co)lex 
(VLL/VLC) usually give tighter bounds after constraint propagation.  
We implement the eight representations and evaluate
them against the subset bounds representation with cardinality and 
variety reasoning.  Results demonstrate that they offer
significantly better pruning and runtime.
\end{abstract}

\section{Introduction}

In constraint programming, we often need to model multisets (or bags) of
objects.  For example, in the template design problem (prob002 in
CSPLib \cite{gent1999CSPLib}),
we need to construct printing
templates, which are multisets of different designs.  Multisets,
unlike sets, can contain repetition of elements.  For popular designs, we may
have multiple copies on the same template.  Surprisingly, whilst there has been
significant progress on developing representations for sets, relatively little
research has been done on how best to represent 
multisets.

Sadler and Gervet \shortcite{sadler2004hybrid} proposed
representing set variables with subset, lexicographic, and cardinality
bounds. Indeed, they suggested that such a
representation could also be used for multisets
\shortcite{sadler2008enhancing}.  However, little detail is
provided about how to do this exactly.  To compare two multisets, they
lexicographically compare their occurrence vectors written in decreasing order.
For instance, $\{3,3,2,1,1\} \preceq \{4\} \preceq \{4,4\}$.  Gervet and Van
Hentenryck \shortcite{gervet2006length} proposed representing set variables
using length-lex bounds, arguing that it provides comparable pruning to the
aforementioned hybrid domains at a fraction of the computational cost.  It
is therefore promising to consider length-lex and related bounds for
multiset variables.  However, as a number of different
orderings are possible, we have undertaken
a theoretical and empirical comparison of the
most promising options. 

As multisets permit repeated elements, we can incorporate information
about the variety (number of distinct elements) \cite{law2009variety} in
addition to the cardinality and position in the lexicographic ordering.  As a
result, we introduce eight different representations for multiset variables in
which we maintain bounds according to one of eight different orderings:
length-(co)lex (LL/LC), variety-(co)lex (VL/VC), length-variety-(co)lex
(LVL/LVC), and variety-length-(co)lex (VLL/VLC) orderings.
These bounds provide
information about the possible cardinality, variety, and position in the
(co)lexicographic ordering of a multiset.  
We evaluate the expressiveness (whether the set of
multisets can be exactly represented) and compactness (whether the interval is
minimal) of the eight representations both theoretically and empirically.
Our results suggest that LVL/LVC and VLL/VLC representations are usually
more expressive and more compact than LL/LC and VL/VC respectively.  The eight
representations give total orderings on multisets, which make enforcing bounds
consistency on multiset variables possible.  
However, when we attempt to enforce bounds consistency
on the bounds of the proposed representations, 
this operation can be NP-hard even on unary constraints.
To test out these representations, we implement the eight
representations and evaluate them against the subset bounds
representation with cardinality and variety reasoning.  Results confirm
that these new representations achieve significantly better pruning and runtime.


\section{Background}

\subsection{Set Variables}

A {\em set\/} is an unordered list of elements {\em without repetition\/}. The
{\em cardinality\/} of a set $S$ is the number of elements in $S$, denoted as
$|S|$.  Gervet \shortcite{Gervet1997} proposed to represent the domain of a set
variable $S$ with an interval $[glb(S), lub(S)]$ such that $D_S = \{m\,|\,glb(S)
\subseteq m \subseteq lub(S)\}$.  The {\em greatest lower bound\/} $glb(S)$
contains all the elements which {\em must exist\/} in the set, while the
{\em least upper bound\/} $lub(S)$ contains any element which {\em can exist\/}
in the set.  $S$ is said to be {\em bound\/} when its lower bound equals its
upper bound (i.e., $glb(S) = lub(S)$).
In this subset bounds representation, the set domain is
ordered partially under $\subseteq$.  
It also neglects the cardinality and the position in 
lexicographic ordering which can be important in many problems.  Thus,
Gervet and Van Hentenryck \shortcite{gervet2006length} proposed to {\em
totally} order a set domain with a length-lex ordering.  This representation
incorporates the cardinality and the position in lexicographic ordering 
directly, giving tighter bounds when enforcing bounds consistency.

\subsubsection{Notation}
Given a universe $U$ of integers $\{1, \dots, n\}$, set variables, denoted as
$S_i$, takes their values from $U$.  Sets are denoted by letters $s$,
$t$, $x$, and $y$. A subset $s$ of $U$ of cardinality $c$ is denoted by
$\{s_1, s_2, \dots, s_c\}$ where $s_1 < s_2 < \cdots < s_c$.

\subsubsection{Length-lex Ordering}
The length-lex ordering $\preceq$ 
totally orders sets first by cardinality and then lexicographically.

\begin{definition} \label{def:ll}
A length-lex ordering $\preceq$ is defined by: \\
$s \preceq t$ iff $s = \emptyset \vee |s| < |t| \vee 
(|s| = |t| \wedge 
(s_1 < t_1 \vee s_1 = t_1 \wedge 
s \setminus \{s_1\} \preceq t \setminus \{t_1\}))$.
\end{definition}

\begin{definition}
Given a universe $U$, a length-lex interval is a pair of sets $\lbr{m, M}$
which represents the sets between $m$ and $M$ in the length-lex ordering (i.e.,
$\{s \subseteq U \,|\, m \preceq s \preceq M\}$).
\end{definition}

Given a universe $U = \{1, \dots, 4\}$, the sets are ordered as follows:
$\emptyset$
$\preceq \br{1} \preceq \br{2} \preceq \br{3} \preceq \br{4}$
$\preceq \br{1,2} \preceq \br{1,3} \preceq \br{1,4}
\preceq \br{2,3} \preceq \br{2,4} \preceq \br{3,4}$
$\preceq \br{1,2,3} \preceq \br{1,2,4} \preceq \br{1,3,4} \preceq \br{2,3,4}$
$\preceq \br{1,2,3,4}$.
The length-lex interval $\lbr{\br{1,2}, \br{3,4}}$ denotes the set 
$\{\br{1,2}, \br{1,3}, \br{1,4}, \br{2,3}, \br{2,4}, \br{3,4}\}$.

\subsection{Multiset Variables}

A {\em multiset\/} is a generalization of set that allows elements to 
repeat. Without loss of generality, we assume that
multiset elements are positive integers from $1$ to $n$.
We shall use $\emptyset$ to denote both the empty set and the empty multiset.
The universe of a multiset is a multiset itself,
which defines the maximum possible occurrences of each element.
Given a universe $U$, we denote a multiset $S$ as
$S = \brac{m_1, m_2, \cdots, m_c}$ where $m_i \leq m_j$ for $1 \leq i \leq j \leq c$,
its cardinality (total number of elements) as $|S|$,
and its variety (total number of distinct
elements) \cite{law2009variety} as $\|S\|$.
For example, if $S = \brac{1,1,2,2,3}$, then $|S| = 5$ and $\|S\| = 3$.
Since an element in a multiset variable can occur multiple
times, we let $occ(i, S)$ be the number of occurrences of
an element $i$ in the multiset $S$.
Walsh \shortcite{walsh2003consistency} proposed using an
occurrence vector $\lbr{occ(1,S), \dots, occ(n,S)}$ to represent a multiset
variable with $n$ elements. 
For example, the occurrence representation for the value
$\brac{1,1,2,2,3}$ with the universe $U = \brac{1,1,2,2,3,3}$ is
$\lbr{occ(1,S), occ(2,S), occ(3,S)} = \lbr{2,2,1}$.

Note that a set value can also be represented using the occurrence
representation in which the number of occurrence is either $0$ or $1$ to denote
the existence of the corresponding element.  
Thus, we adopt the occurrence representation for multiset variables and order
the occurrence vector to give various orderings in multisets.

\section{Lex-induced Orderings in Multisets}
\label{sec:lex}

The length-lex representation for sets incorporates information about
the length (cardinality) and position in the lexicographic ordering.  Such
a representation can be extended to include the variety information since
multisets allow repeated elements.  This gives a total of eight different ways
to order multisets.  In the following, we formally define the eight
orderings, in which four of them order the position lexicographically and the
other four colexicographically.

\subsection{Lex Orderings}

The {\em lex ordering\/} $\preceq_l$ totally orders multisets
{\em lexicographically\/}.  Here, we assume the multisets are represented by
the occurrence representation (i.e., the number of occurrences of each element
are stored in an occurrence vector).  Thus, given two multisets $x$ and $y$, we
compare their occurrence vectors $\lbr{occ(1,x), \dots, occ(n,x)}$ and
$\lbr{occ(1,y), \dots, occ(n,y)}$ from the {\em first} position to the
{\em last}.

\begin{definition}
A {\em lex ordering\/} $\preceq_l$ is defined by: \\
$x \preceq_l y$ iff $(x = y) \vee
(\exists i, occ(i, x) < occ(i, y) \wedge 
\forall j < i, occ(j, x) = occ(j, y))$.
\end{definition}

For example, consider two multisets 
$x = \brac{1,2,2}$ and $y = \brac{1,3,3}$.
Their occurrence vectors are $\lbr{1,2,0}$ and $\lbr{1,0,2}$ respectively.
$\brac{1,3,3} \preceq_l \brac{1,2,2}$ because $occ(1,y) = occ(1,x)$ and
$occ(2,y) < occ(2, x)$.

\subsection{Colex Orderings}

Contrary to the lex ordering, the {\em colex ordering\/} $\preceq_c$
compares the occurrence vectors of two multisets from the {\em last} position
to the {\em first}.

\begin{definition}
A {\em colex ordering\/} $\preceq_c$ is defined by: \\
$x \preceq_c y$ iff $(x = y) \vee
(\exists i, occ(i, x) < occ(i, y) \wedge 
\forall j > i, occ(j, x) = occ(j, y))$.
\end{definition}

For example, let two multisets 
$x = \brac{1,3,3}$ and $y = \brac{2,3,3}$ with 
occurrence vectors $\lbr{1,0,2}$ and $\lbr{0,1,2}$ respectively.
They are ordered as $\brac{1,3,3} \preceq_c \brac{2,3,3}$ because
$occ(3,x) = occ(3,y)$ and $occ(2,x) < occ(2, y)$.

\subsection{Induced Orderings}

Given a total order $\preceq_\beta$ on a set of multisets, we can have four 
different $\preceq_\beta$-induced orderings when we integrate $\preceq_\beta$ 
with cardinality and/or variety of multisets.

\subsubsection{Length-$\beta$ Ordering}

The {\em length-$\beta$\/} {\em ordering\/} $\preceq_{l\beta}$ totally orders
multisets first by their cardinality, and then by the $\beta$ ordering:
$x \preceq_{l\beta} y$
iff $|x| < |y| \vee (|x| = |y| \wedge x \preceq_\beta y)$.

\subsubsection{Variety-$\beta$ Ordering}

The {\em variety-$\beta$\/} {\em ordering\/} $\preceq_{v\beta}$ totally orders
multisets first by their variety, and then by the $\beta$ ordering:
$x \preceq_{v\beta} y$ iff
$\|x\| < \|y\| \vee (\|x\| = \|y\| \wedge x \preceq_\beta y)$.

Length-$\beta$ and variety-$\beta$ prefer cardinality and variety over the
$\beta$ ordering respectively.  In fact, both cardinality and variety can be
considered together, giving two more orderings.

\subsubsection{Length-variety-$\beta$ Ordering}

The {\em length-variety-$\beta$\/} {\em ordering\/} $\preceq_{lv\beta}$ totally
orders multisets first by their cardinality, then by their variety, and then
by the $\beta$ ordering:
$x \preceq_{lv\beta} y$ iff $|x| < |y| \vee (|x| = |y| \wedge 
x \preceq_{v\beta} y)$.  

\subsubsection{Variety-length-$\beta$ Ordering}

The {\em variety-length-$\beta$\/} {\em ordering\/} $\preceq_{vl\beta}$ totally
orders multisets first by their variety, then by their cardinality, and then
by the $\beta$ ordering:
$x \preceq_{vl\beta} y$ iff $\|x\| < \|y\| \vee
(\|x\| = \|y\| \wedge x \preceq_{l\beta} y)$.

Since lex and colex orderings are total orders, we can have eight different
orderings by substituting $\beta$ by the lex and colex orderings.
For example, substituting $\beta$ by the lex ordering in the length-$\beta$
ordering gives the length-lex ordering LL ($\preceq_{ll}$).
Similarly, we can have 
variety-lex VL ($\VL$), 
length-variety-lex LVL ($\LVL$),
variety-length-lex VLL ($\VLL$), 
length-colex LC ($\LC$), 
variety-colex VC ($\VC$),
length-variety-colex LVC ($\LVC$), and
variety-length-colex VLC ($\VLC$) orderings.

The above eight orderings are applicable to multisets.
All the four colex orderings on multisets reduce to the
LL ordering on sets introduced by Gervet and 
Van Hentenryck \shortcite{gervet2006length}.
Note that, when we consider a fixed length, 
the colex ({\em resp.~\/}lex) ordering for set values is equivalent to
ordering the occurrence vector lexicographically ({\em
resp.~\/}colexicographically).

\begin{table*}[ht]\centering
\caption{The four lex orderings 
for the domain of a multiset variable $S$ with universe $U = \brac{1,2,2,3,3}$}
\scalebox{0.9}{
\begin{tabular}{|l|l|}
\hline
\multirow{2}{*}{Length-lex (LL)} &
$\emptyset \preceq_{ll} \brac{3} \preceq_{ll} \brac{2} \preceq_{ll} \brac{1}
\preceq_{ll} \brac{3,3} \preceq_{ll} \brac{2,3} \preceq_{ll} \brac{2,2}
\preceq_{ll} \brac{1,3}$ \\
& $\preceq_{ll} \brac{1,2} \preceq_{ll} \brac{2,3,3} \preceq_{ll} \brac{2,2,3}
\preceq_{ll} \brac{1,3,3} 
\preceq_{ll} \brac{1,2,3} \preceq_{ll} \brac{1,2,2}$ \\
& $\preceq_{ll} \brac{2,2,3,3} \preceq_{ll} \brac{1,2,3,3} \preceq_{ll} 
\brac{1,2,2,3} \preceq_{ll} \brac{1,2,2,3,3}$ \\
\hline
\multirow{2}{*}{Variety-lex (VL)} &
$\emptyset \preceq_{vl} \brac{3} \preceq_{vl} \brac{3,3} \preceq_{vl} \brac{2}
\preceq_{vl} \brac{2,2} \preceq_{vl} \brac{1}
\preceq_{vl} \brac{2,3} \preceq_{vl} \brac{2,3,3}$ \\
& $\preceq_{vl} \brac{2,2,3} \preceq_{vl} \brac{2,2,3,3} 
\preceq_{vl} \brac{1,3} \preceq_{vl} \brac{1,3,3}
\preceq_{vl} \brac{1,2} \preceq_{vl} \brac{1,2,2}$ \\
& $\preceq_{vl} \brac{1,2,3} 
\preceq_{vl} \brac{1,2,3,3} \preceq_{vl} \brac{1,2,2,3} \preceq_{vl}
\brac{1,2,2,3,3}$ \\
\hline
\multirow{2}{*}{Length-variety-lex (LVL)} &
$\emptyset \preceq_{lvl} \brac{3} \preceq_{lvl} \brac{2} \preceq_{lvl} \brac{1}
\preceq_{lvl} \brac{3,3} \preceq_{lvl} \brac{2,2}
\preceq_{lvl} \brac{2,3} \preceq_{lvl} \brac{1,3}$ \\
& $\preceq_{lvl} \brac{1,2} \preceq_{lvl} \brac{2,3,3} 
\preceq_{lvl} \brac{2,2,3} \preceq_{lvl} \brac{1,3,3} 
\preceq_{lvl} \brac{1,2,2} \preceq_{lvl} \brac{1,2,3}$ \\
& $\preceq_{lvl} \brac{2,2,3,3} \preceq_{lvl} \brac{1,2,3,3} 
\preceq_{lvl} \brac{1,2,2,3} \preceq_{lvl} \brac{1,2,2,3,3}$ \\
\hline
\multirow{2}{*}{Variety-length-lex (VLL)} & 
$\emptyset \preceq_{vll} \brac{3} \preceq_{vll} \brac{2} \preceq_{vll} \brac{1}
\preceq_{vll} \brac{3,3} \preceq_{vll} \brac{2,2}
\preceq_{vll} \brac{2,3} \preceq_{vll} \brac{1,3}$ \\
& $\preceq_{vll} \brac{1,2} 
\preceq_{vll} \brac{2,3,3} \preceq_{vll} \brac{2,2,3}
\preceq_{vll} \brac{1,3,3}
\preceq_{vll} \brac{1,2,2}$ \\
& $\preceq_{vll} \brac{2,2,3,3}
\preceq_{vll} \brac{1,2,3} \preceq_{vll} \brac{1,2,3,3} \preceq_{vll} 
\brac{1,2,2,3} \preceq_{vll} \brac{1,2,2,3,3}$ \\
\hline
\end{tabular}}
\label{tab:eg}
\end{table*}

The domain of a multiset variable is simply a set of multisets.  We can thus
totally order the domain values of a variable according to the eight orderings.
To illustrate the differences, 
Table \ref{tab:eg} lists the domain of a multiset variable $S$ with universe
$U = \brac{1,2,2,3,3}$ in the four lex orderings.
Take the 
LVL ordering as an example.
We first order the multisets by their cardinality.  Thus, $\emptyset$ has
cardinality $0$ and is the first multiset, followed by the multisets with
cardinalities $1$, $2$, and so on.  For multisets of the same cardinality, we
then compare their variety.  Consider the segment with cardinality $2$, i.e.,
from $\brac{3,3}$ to $\brac{1,2}$.  The multisets $\brac{3,3}$ and $\brac{2,2}$
are ordered before $\brac{2,3}$, $\brac{1,3}$, and $\brac{1,2}$ because the
former two have variety $1$ and the latter ones have variety $2$.  Lastly, we
order the multisets lexicographically.  The occurrence vectors of $\brac{3,3}$
and $\brac{2,2}$ are $\lbr{0,0,2}$ and $\lbr{0,2,0}$ respectively.  Thus,
$\brac{3,3} \preceq_{lvl} \brac{2,2}$ because
$occ(1,\brac{3,3}) = occ(1, \brac{2,2}) = 0$ and $occ(2, \brac{3,3}) < occ(2,
\brac{2,2})$.

Given a multiset variable, we can approximate its domain, which is a set $S$ of
multisets, with an {\em $\alpha$-interval\/}, where $\alpha$ refers to one of
the above eight orderings.
The interval $\lbr{m, M}_\alpha$ must contain all the multisets in $S$
such that $m$ and $M$ are the lower and upper bounds of $S$ respectively.
We also define the {\em $\alpha$-closure\/} of $S$ which is the minimal possible
$\alpha$-interval containing $S$.

\begin{definition} \label{def:rep}
Given an $\alpha$ ordering, an {\em $\alpha$-interval\/} $\lbr{m, M}_\alpha$
is a set of multisets defined by
$\lbr{m, M}_\alpha = \{x \,|\, m \preceq_{\alpha} x \preceq_{\alpha}
M\}$.  The {\em $\alpha$-closure} of $S$ 
is defined by $cl_{\alpha}(S) = \lbr{m, M}_{\alpha}$, where
$S \subseteq \lbr{m, M}_{\alpha}$ and there does not exist
$m \prec_\alpha m'$ and $M'
\prec_\alpha M$ such that ($m \neq m'$ or $M \neq M'$) and
 $S \subseteq \lbr{m', M'}_\alpha$.
\end{definition}

\begin{definition}
An {\em $\alpha$ representation\/}
of a set $S$ of multisets is 
$cl_{\alpha}(S)$.
An $\alpha$ representation of $S$ is {\em exact} if $S = cl_{\alpha}(S)$.
\end{definition}

For example, let the universe $U = \brac{1,2,2,3,3}$ and
$S = \{\brac{1}, \brac{2,2}, \brac{2,3}\}$.
The $lvl$-closure of $S$ 
is the
$lvl$-interval $\lbr{\brac{1}, \brac{2,3}}_{lvl}$.
This representation is {\em not} exact, as the interval contains the multiset
$\brac{3,3} \notin S$.


\section{Expressiveness}

An exact representation gives the tightest possible
bounds and contains no undesired values.
It is often the case that a set of multisets can be exactly
represented using one representation but not using a different
representation.
In this section, we compare the eight representations to see which ordering
is better in terms of the notion ``expressiveness''.

\begin{definition} \cite{walsh2003consistency}
Given a universe $U$ and 
two different multiset representations $A$ and $B$. 
$A$ is said to be {\em as expressive as\/} $B$ if\, $\forall S \subseteq U, 
(S = cl_A(S)) \leftrightarrow (S = cl_B(S))$.
$A$ is said to be {\em more expressive\/} than $B$ if 
$\forall S \subseteq U, 
(S = cl_B(S)) \rightarrow (S = cl_A(S))$
and 
$\exists S \subseteq U, (S = cl_A(S)) \wedge (S \neq cl_B(S))$.
$A$ and $B$ are {\em incomparable\/} if
neither one of them is more expressive than the other.
\end{definition}

The following propositions compare the expressiveness of the eight
representations under the conditions that the cardinality and/or variety of
a set of multisets is fixed.

\begin{proposition}\label{thm:cvfix}
When both the cardinality and variety are fixed,
(i) the LVL/LVC 
    representation is as expressive as the VLL/VLC 
    representation,
(ii) the LVL/LVC 
     and VLL/VLC 
     representations are more expressive than the LL/LC 
     and VL/VC 
     representations respectively, and
(iii) the LVL 
      is as expressive as the LVC 
      and the VLL 
      is as expressive as the VLC. 
      
\end{proposition}

The results in Proposition \ref{thm:cvfix} can be demonstrated using the example
in Table \ref{tab:eg}. When the cardinality and variety
are $2$ and $1$ respectively, 
the LVL 
and VLL 
representations can
exactly represent $\{\brac{2,2}, \brac{3,3}\}$ by 
the $lvl$-interval $\lbr{\brac{3,3}, \brac{2,2}}_{lvl}$ and the $vll$-interval
$\lbr{\brac{3,3}, \brac{2,2}}_{vll}$ respectively.
However, the LL 
and VL 
representations give the
$ll$-interval $\lbr{\brac{3,3}, \brac{2,2}}_{ll}$ and the $vl$-interval 
$\lbr{\brac{3,3}, \brac{2,2}}_{vl}$ respectively,
in which both contain the additional undesired value $\brac{2,3}$.

The following two propositions relax the conditions to
the case that either the
cardinality or the variety is fixed.

\begin{proposition}
When the cardinality is fixed,
(i) the LVL/LVC 
    representation is more expressive than the VLL/VLC, 
    LL/LC, 
    and VL/VC 
    representations, and
(ii) the LL 
     representation is as expressive as the LC 
     representation.
\end{proposition}

\begin{proposition}
When the variety is fixed,
(i) the VLL/VLC 
    representation is more expressive than the LVL/LVC, 
    LL/LC, 
    and VL/VC 
    representations, and
(ii) the VL 
     representation is as expressive as the VC 
     representation.
\end{proposition}

In Table \ref{tab:eg}, when the cardinality is $3$, the LVL 
representation can exactly represent the multisets by the $lvl$-interval
$\lbr{\brac{2,3,3}, \brac{1,2,3}}_{lvl}$, 
while the VLL, 
LL, 
or VL 
representations cannot.  
There are additional undesired values in their corresponding intervals.
In fact, when only the variety is fixed, we obtain similar results.
Suppose the variety is $2$, the VLL representation can exactly represent the
multisets by the $vll$-interval $\lbr{\brac{2,3}, \brac{2,2,3,3}}_{vll}$,
while the LVL, LL, or VL representations cannot.

\section{Compactness}

The notion of expressiveness concerns the exactness of the
representation. However, a domain $D$ of a
multiset variable might not be exactly represented using any of the eight
representations, i.e., $D \subset cl_{\alpha}(D)$.
In such cases, $cl_{\alpha}(D)$ is an approximation that contains some
undesired values, and our
expressiveness notion does not apply.
In this section, we define a new notion called {\em compactness} to compare
the eight representations.
This definition is based on a comparison of the size of the domains, and is
different from the notion of dominance which is based on the size of search
tree \cite{jefferson2007representions}.

\begin{definition}
Given a universe $U$ and
two different multiset representations $A$ and $B$. 
$A$ is 
{\em as compact as\/} $B$ if
$\forall S \subseteq U, |cl_A(S)| = |cl_B(S)|$.
$A$ is 
{\em more compact\/} than $B$ if 
$\forall S \subseteq U, |cl_A(S)| \leq |cl_B(S)|$
and
$\exists S \subseteq U, |cl_A(S)| < |cl_B(S)|$.
$A$ and $B$ are {\em compactly incomparable\/} if neither one of them is
more compact than the other.
\end{definition}

The following proposition characterizes the compactness of the eight
orderings.

\begin{proposition}
(i) The LVL/LVC 
    representation is more compact than the LL/LC 
    representation and compactly incomparable to the VLL/VLC 
    representation.
(ii) The VLL/VLC 
     representation is more compact than the VL/VC 
     representation.
(iii) The LL/LC 
      representation is compactly incomparable to the VL/VC 
      representation.
\end{proposition}

In Table \ref{tab:eg},
suppose we want to represent the set $S$ of all multisets whose variety is $2$.
Both the LVL 
and LL 
representations cannot exactly represent $S$ and give a $\alpha$-closure with
the same lower and upper bounds (i.e., $\brac{2,3}$ and $\brac{2,2,3,3}$
respectively). 
Both $lvl$- and $ll$-intervals contain undesired
values.  By comparing their compactness,
$|cl_{lvl}(S)| = 9 < |cl_{ll}(S)| = 10$.
The LVL 
representation is more compact than LL 
representation.

Using the VL/VLL representations for multiset variables would be useful when we
have tight constraints on the varieties of the multiset variables.  For
instance, Law, Lee, and Woo \shortcite{law2009variety} demonstrated the value
of this on extended Steiner system problems in which there are tight
constraints over the varieties.
On the other hand, the LL/LVL representations would favour the kind of problems
with more cardinality restrictions or with variables having fixed
cardinalities.

\section{Empirical Comparisons}
\label{sec:comp}

Before we apply the eight representations to model and solve multiset
problems, we first empirically evaluate their expressiveness and compactness.
We perform experiments to compare the size of the eight representations
of a set $D$ of multisets
when different cardinality and variety constraints are imposed.
In the experiment, the universe $U$ is a multiset which contains 10 occurrences
of elements 1 to 5.  For all instances, $D$ is a randomly generated subset of 
the power set of $U$.
The comparison aims at measuring the compactness of different representations 
in approximating $D$.
We record $|cl_{\alpha}(D)|$, the number of multisets in the
$\alpha$-closure of $D$ 
that satisfies the cardinality and the variety constraints, where $\alpha$
refers to the eight representations:
LL, LC, VL, VC, LVL, LVC, VLL, and VLC.
Due to space limitation, we summarize the observations as follows.

When both cardinality and variety are fixed, the
LVL/LVC and VLL/VLC representations
can always exactly represent the domain values, giving the
corresponding minimal $\alpha$-interval $cl_{\alpha}(S)$.
For all instances, the LVL/LVC and VLL/VLC representations demonstrate a large
reduction in the domain size when compared with the LL/LC and VL/VC
representations.

When the variety is fixed, the VLL/VLC ordering first considers the variety of each
multiset and narrows down the bounds to a larger extent by removing
the multisets with unwanted varieties.  For each variety, the multisets are
then ordered by their cardinality, which allows further pruning of the multisets
with undesired cardinalities on the domain bounds.  Thus, the VLL/VLC representation
can always give the exact representation and
achieve on average one to two orders of
magnitude reduction in the domain size when compared with the LL/LC and
VL/VC representations.
In contrast, the LVL/LVC representation can always give the exact representation
when the cardinality is fixed.

When the cardinality and variety are constrained to certain ranges, although all
eight representations fail to give the exact representation for all
instances, the LVL/LVC and VLL/VLC representations
are more compact than the LL/LC and VL/VC representations respectively.

To conclude, the LVL/LVC and VLL/VLC representations
are always more compact than the LL/LC and VL/VC respectively.
This means that they will usually give tighter bounds during constraint
propagation.
In the following, we study how the eight representations behave in practice
as bounds propagation in a multiset solver.

\section{Bounds Consistency}
\label{sec:bc}

Since a multiset domain is totally ordered in the eight representations,
we can enforce bounds consistency.
To be more precise, we define bounds consistency on a $k$-ary constraint on
multiset variables (for any $k$).

\begin{definition} Bounds Consistency (BC) \\
Let $S_1, \dots, S_n$ be multiset variables with interval domains $D(S_i)
= \lbr{m_{S_i}, M_{S_i}}$. 
Given a constraint ${\cal C}$ over $S_1,\dots,S_n$ and an $\alpha$ ordering, a
value $m_i$ for variable $S_i$ has an {\em $\alpha$-bound support\/} $(m_1,
\dots, m_n)$ if the support satisfies ${\cal C}$ and $\forall m_i, m_{S_i}
\preceq_{\alpha} m_i \preceq_{\alpha} M_{S_i}$.

The constraint ${\cal C}$ 
is {\em bounds consistent\/} iff for each $S_i$, both $m_{S_i}$ and $M_{S_i}$
have $\alpha$-bound supports.
\end{definition}

The eight representations offer greater expressiveness, but we have to be
careful that reasoning remains tractable.  Indeed, even with a single unary
constraint, we can get intractability. 

\begin{theorem} \label{thm:unary}
There exists a constraint on one set variable such that enforcing BC on
subset bounds is polynomial but enforcing BC on LL 
bounds is NP-hard.
\end{theorem}

\ignore{
\begin{proof} (sketch)
Let the possible values of a set variable $S$ be $\{1,2,\dots,2n-1,2n\}$.
Define $x_i$ to be true iff the value $2i \in S$ and
$x_i$ to be false iff the value $2i-1 \in S$ for $i = 1,\dots,n$.

We define the unary constraint $C$ so that it is satisfied iff
$|S| = n$ and ($2i \in S$ iff $2i-1 \notin S$ for $i = 1,\dots,n$) and the
corresponding assignment of variables $x_i$'s satisfies a SAT formula.
$\exists S_1, S_2 \subseteq S$ such that $S_1 = glb(S) \wedge S_2 = lub(S)$.
Enforcing BC on $C$ with $S$ represented by subset bounds 
$D(S) = \{m \,|\, S_1 \subseteq\ m \subseteq S_2\}$ does not do pruning.
However, enforcing BC on the LL 
representation of $D(S)$ solves SAT
because the lower bound of $S$ indicates if the SAT formula is satisfiable.
\end{proof}
}

\begin{proof}
Reduction from 3-SAT with $N$ variables, $X_1$ to
$X_N$ and $M$ clauses. We construct
a set variable $S$ with elements that have the following meaning:
$2i$ represents a truth assignment in which $X_i$ is true whilst
$2i-1$ represents a truth assignment in which $X_i$ is
false ($1 \leq i \leq N$), and each integer above
$2N$ represents one of the (polynomial number of)
distinct clauses. We consider an unary constraint on
this set variable which is satisfied only when
the set contains integers representing a proper truth assignment (that is,
$2i \in $S iff $2i-1 \not\in S$ for $1 \leq i \leq N$) and this assignment
satisfies the clauses represented by the integers in
the set greater than $2N$, or the set contains integers representing
a superset of a proper truth assignment (that is,
either $2i$ or $2i-1$ or both occur in $S$ for $1 \leq i \leq N$).
Subset bounds are polynomial to compute since, if the
upper bound includes a proper truth assignment, we leave
the upper bound untouched and adjust the lower bound to include
any necessary elements in linear time and, where needed, check
the truth assignment. On the other hand, if
the upper bound does not include a proper truth assignment, the
unary constraint has no support. By comparison, length-lex bounds are
NP-hard to compute. We consider
domains that fix the possible and necessary 
elements to be the clause that we wish to decide, and make
none of the other integers necessary but all of them possible. Then,
enforcing bound consistency on the length-lex bounds
will allow us to decide the satisfiability of the original
formula.
\end{proof}

It is worth noting that the opposite does not hold.  
If LL 
bounds are polynomial to compute, then subset bounds are too.

\ignore{
Theorem \ref{thm:unary} is 
similar in idea with Yip and Van Hentenryck's 
work \shortcite{yip2010exponential}.
}

\begin{theorem} \label{thm:nary}
Given an $n$-ary constraint on set and/or multiset variables. If enforcing BC
on LL 
bounds is polynomial, then enforcing BC on subset bounds is also polynomial.
\end{theorem}

\begin{proof} (sketch)
Let the possible values of a set variable $S$ be $\{1,\dots,n\}$.  We can
convert subset bounds into LL 
bounds easily by ordering the sets first
by cardinality and then lexicographically.  This operation is polynomial.
After enforcing BC on LL 
bounds, we can then convert LL 
bounds back to subset bounds using the inclusion propagator
\cite{gervet2006length}.
Such conversion is also polynomial.  Thus, if enforcing BC on LL 
bounds is polynomial, then enforcing BC on subset bounds is also polynomial.
\end{proof}

With two unary constraints, Sellmann's Lemma 1 shows that finding the
fixpoint on the LL 
representation of a single set variable is NP-hard \cite{sellmann09decomposing}.
Given the above theorems, enforcing BC on LL 
bounds is NP-hard. However, exponential-time propagation algorithms may still
help reduce runtimes \cite{yip2010exponential}.

Here, we show an example on how BC works on the domains in the
LL and LVL representations.

Given the universe $U = \brac{1,1,1,2,2,2,3,3,3}$ and multiset
variables $X$, $Y$, and $Z$.  The constraints are:
$|X| = |Y| = |Z| = 3$, $\|Z\| = 1$, and $X \cap Y = Z$.
The initial domains are
$D(X) = D(Y) = D(Z) = \lbr{\emptyset, \brac{1,1,1,2,2,2,3,3,3}}_{lvl}$.
In LVL representation, 
enforcing $|X| = |Y| = |Z| = 3$ tightens the bounds to have cardinality 3,
i.e., $D(X) = D(Y) = D(Z) = \lbr{\brac{3,3,3}, \brac{1,2,3}}_{lvl}$.
The bounds corresponds to the occurrence vectors $\lbr{0,0,3}$ and $\lbr{3,0,0}$.
Since $\|\brac{1,2,3}\| \neq 1$, the upper bound of $Z$ is updated to
$\brac{1,1,1}$, resulting $D(Z) = \lbr{\brac{3,3,3}, \brac{1,1,1}}_{lvl}$.
This triggers the propagation on $X \cap Y = Z$ and tightens the upper bounds
of $X$ and $Y$.  After constraint propagation,
$X = Y = \lbr{\brac{3,3,3}, \brac{1,1,1}}_{lvl}$.
Now, the problem is bounds consistent and
$|D(X)| = |D(Y)| = |D(Z)| = 3$.
However, in the LL representation, the problem is bounds consistent
after enforcing the cardinality constraint $|X| = |Y| = |Z| = 3$.
$D(X) = D(Y) = D(Z) = \lbr{\brac{3,3,3}, \brac{1,1,1}}_{ll}$ and
$|D(X)| = |D(Y)| = |D(Z)| = 10$.
Thus, different representations result in different domain size after enforcing
BC, and LVL gives a tighter bound than LL in this example.

\section{Experimental Results}

To verify the feasibility and efficiency of our proposal, we adapt and simplify
the implementation of the length-lex representation for set variables
\cite{yip2008bound} to implement the eight representations 
(LL, LVL, VL, VLL, LC, LVC, VC, VLC) 
for multiset variables in ILOG Solver 6.0 \cite{ilogsolver03}.  
We have also developed the ternary intersection ($X \cap
Y = Z$) and unionplus ($X \uplus Y = Z$) multiset constraints, which are not
available in the original LL implementation.

\begin{table*}[ht]\centering
\caption{Experimental results of the extended Steiner system.}
\scalebox{0.85}{
\begin{tabular}{|c|rr|rr|rr|rr|rr|}
\hline
    & \multicolumn{ 2}{c|}{SB+CR+VR} & \multicolumn{ 2}{c|}{LL} & \multicolumn{ 2}{c|}{LVL} & \multicolumn{ 2}{c|}{VL} & \multicolumn{ 2}{c|}{VLL} \\
\hline
   $t$,$k$,$u$,$b$,$v$ &       Fail &       Time &       Fail &       Time &       Fail &       Time &       Fail &       Time &       Fail &       Time \\
\hline




   2,4,5,4,2 &      57329 &       3.59 &      19187 &       1.48 & {\bf 2930} & {\bf 0.34} &       3790 &      95.37 &       2945 &       3.38 \\

   2,4,5,5,2 &     356785 &      28.71 &      89768 &      10.04 & {\bf 19718} & {\bf 3.13} &      30755 &     541.13 &      19991 &      14.32 \\


   3,4,4,4,2 &       1710 &        0.1 &        942 &       0.08 &  {\bf 278} & {\bf 0.03} &        309 &       1.77 &        305 &       0.58 \\

   3,4,4,5,2 &      30034 &       2.36 &      13541 &       1.39 &  {\bf 658} & {\bf 0.11} &        922 &      20.33 &        729 &      15.13 \\




   3,4,5,5,3 &     312397 &      22.17 &      38109 &       5.84 & {\bf 12195} & {\bf 1.36} &          - &          - &      12363 &       7.23 \\

   3,4,5,6,3 &    2108410 &     190.15 &     281911 &      57.83 & {\bf 103163} & {\bf 13.39} &          - &          - &     106145 &      63.83 \\

   3,4,5,7,3 &    9813128 &       1097 &    1352165 &     380.42 & {\bf 384145} & {\bf 63.05} &          - &          - &     398511 &     285.16 \\
\hline
\end{tabular}}
\label{tab:expt-steiner}
\end{table*}

\begin{table*}[ht]\centering
\caption{Experimental results of the generalized social golfer problem.}
\scalebox{0.85}{
\begin{tabular}{|c|rr|rr|rr|rr|rr|}

\hline
           & \multicolumn{ 2}{c|}{SB+CR+VR} & \multicolumn{ 2}{c|}{LL} & \multicolumn{ 2}{c|}{LVL} & \multicolumn{ 2}{c|}{VL} & \multicolumn{ 2}{c|}{VLL} \\
\hline
$w$,$m$,$n$,$g$,$p$,$v$ &       Fail &       Time &       Fail &       Time &       Fail &       Time &       Fail &       Time &       Fail &       Time \\
\hline
3,3,3,2,4,2 &      14934 &       1.61 &      15108 &       0.94 &      14479 &       0.87 & {\bf 2171} &       0.44 &       2395 & {\bf 0.27} \\

3,3,4,2,4,2 &     394570 &      40.29 &     111102 &       6.41 &     103756 &       5.59 &   {\bf 39} &       0.06 &   {\bf 39} & {\bf 0.05} \\

3,3,4,2,5,2 &     185839 &      20.32 &     181801 &      12.37 &     172818 &      11.27 & {\bf 11536} &       8.61 &      12428 & {\bf 2.84} \\


4,3,4,2,4,2 &          - &          - &   14071439 &    1003.03 &   12983736 &     874.96 & {\bf 151132} &      78.47 & {\bf 151132} & {\bf 41.6} \\

4,3,4,2,5,2 &          - &          - &   12818684 &       1103 &   12496315 &    1046.14 & {\bf 1035895} &     437.89 &    1098395 & {\bf 173.74} \\

3,4,3,2,4,3 &    2631024 &     348.04 &    1889782 &     129.28 &    1510939 &      94.21 &   {\bf 21} & {\bf 0.28} &   {\bf 21} &       0.29 \\

3,4,4,2,4,3 &          - &          - &    4062535 &     280.02 &    3339400 &     210.61 &   {\bf 27} &       3.99 &   {\bf 27} & {\bf 3.95} \\
\hline
\end{tabular}}
\label{tab:expt-golfer}
\end{table*}

We perform experiments on the extended Steiner system and the
generalized social golfer problem.
They are run on a Sun Blade 2500 ($2 \times 1.6$GHz US-IIIi)
workstation with 2GB memory.  
We report the number of fails (i.e., the number of backtracks occurred in
solving a model) and CPU time in seconds to find and prove the optimal solution
for each instance.
Comparisons are made among the subset bounds representations with
cardinality-variety reasoning (SB+CR+VR) \cite{law2009variety}
and the eight representations we have implemented.  
Since the results of the four colex representations (LC, LVC, VC, VLC) are
similar to their corresponding lex counterparts (LL, LVL, VL, VLL), they are
not reported in the tables.
In the tables, the first column shows the problem instances. 
The subsequent columns show the results of using various representations.
The best number of fails and CPU time among the results for each
instance are highlighted in bold.
A cell labeled with ``-'' denotes a timeout after 20 minutes.


The extended Steiner system $ES(t,k,u,b)$, 
an important and practical multiset problem in
information retrieval \cite{Johnson1972,Bennett1980,Park2008},
is a collection of $b$ blocks.
Each block is a $k$-element multiset drawn from a $u$-element set whose elements
can be drawn multiple times.  For every two blocks in the collection, the
cardinality of their intersection must be smaller than $t$.
We adapt the problem to become an optimization problem which maximizes the
sum of the varieties of the multisets. 
To further increase difficulty, we constrain each multiset variable to have
variety at least $v$.


The generalized social golfer problem $SG(w,m,n,g,p)$ extends the
social golfer problem (prob010 in CSPLib \cite{gent1999CSPLib}) 
from sets to multiset, in which we
schedule $m$ teams of $n$ members to $g$ groups of $p$ golfers over $w$ weeks.
Each group contains golfers from different teams and they play against each
other.  To maximize the socialization, the number of times two teams meet with
each other again is minimized.  Similar to the extended Steiner system, each
multiset variable is constrained to have variety at least $v$.

Tables \ref{tab:expt-steiner} and \ref{tab:expt-golfer}
show the experimental results of the extended Steiner system and the
generalized social golfer problem respectively.
All the four lex representations give fewer number of fails and
faster runtime than the SB+CR+VR \cite{law2009variety}.
This confirms that the lex representations
take advantage of the cardinality and variety information to give tighter
bounds than the SB+CR+VR.

In the extended Steiner system, the LVL representation always achieves the
fewest number of fails.
There is about a 95\% reduction in the number of fails when compared to the
SB+CR+VR.  The LVL representation achieves fewer number of fails
than the VLL representation because the problem has tighter constraints on the
cardinalities than the varieties of the multiset variables.

When comparing the results between LL and LVL, the latter
performs better.  This is because in the LVL representation,
the multisets are ordered according to their varieties under the same
cardinality.  When enforcing BC, the multisets with the same
varieties can be pruned together when they violate the variety constraints.
However, in the LL representation, these multisets are scattered over the
ordering and we cannot remove all of them from the domain at the same
time, thus resulting in a larger search tree and number of fails.
Similarly, VLL performs better than VL.

The instances listed in Table \ref{tab:expt-steiner} are all satisfiable.
In our experiments, there are some unsatisfiable instances, in which
the number of fails and runtime of LVL and VLL
can be slightly larger than LL and VL respectively.
We also tried to fix both cardinalities and varieties of the multiset
variables.  Since the multisets are ordered lexicographically under a fixed
cardinality and variety, LVL and VLL give the same number of
fails.

For the generalized social golfer problem, VL and VLL 
perform better than LL and LVL 
because the problem has tighter constraints on the varieties than
the cardinalities of the multiset variables.  Since there are much more
constraints in the problem when compared to those in the extended Steiner system,
the generalized social golfer problem is more complicated.  We observe that
the VL representation always achieves the fewest number of fails.  However,
the VLL representation has the fastest runtime because the extra prunings in the
VL representation cannot compensate the overhead in finding new bounds of
multiset variables.

\section{Conclusion}

We have proposed eight representations for multiset variables, which integrate
together information about the cardinality, variety, and position in the
(co)lexicographic ordering.  We have made a detailed
comparison of the expressiveness and compactness
between the eight different representations. The LVL/LVC and VLL/VLC
representations are always more expressive and more compact
than the LL/LC
and VL/VC representations.  Compactness is a new notion which
lets us compare inexact representations. 
We have also performed experiments on some benchmark problems.  Experimental
results confirm that LVL and VLL usually give tighter bounds
during constraint propagation, resulting in smaller search trees 
and better runtimes. 
In some cases, LVL performs better, and sometimes VLL.
It would be interesting to study if the two representations can be linked
together so that we can take advantage of each representation.

\ignore{
For future work, Cheng {\em et al.\/} \shortcite{cheng99increasing} suggested to
combine different models of the same problem using channeling constraints. The
combined model can take advantages of each sub-model to increase constraint
propagation and efficiency.  Since multiset variables can be represented
represented using the SB+CR+VR \cite{law2009variety} and our
proposed lex representations, it would be interesting to study if the two
representations can be combined to achieve better results.
}

\section{Acknowledgments}

We thank the anonymous referees for constructive comments.  The work described
in this paper was substantially supported by grants (CUHK413808 and CUHK413710)
from the Research Grants Council of Hong Kong SAR.
Toby Walsh is funded by the Australian Department of Broadband, Communications
and the Digital Economy, the ARC, and the Asian Office of Aerospace Research
and Development through grant AOARD-104123.

\bibliography{ref}
\bibliographystyle{aaai}

\end{document}